\documentclass[conference]{IEEEtran}
\usepackage{calc}
\usepackage{amsmath,amssymb,amsfonts}
\usepackage{algorithmic}
\usepackage{graphicx}
\usepackage{textcomp}
\usepackage{xcolor}
\def\BibTeX{{\rm B\kern-.05em{\sc i\kern-.025em b}\kern-.08em
    T\kern-.1667em\lower.7ex\hbox{E}\kern-.125emX}}

\usepackage{amsthm}
\usepackage{xspace}
\usepackage{hyperref}
\usepackage{url}    
\usepackage{bbm}

\usepackage[style=numeric,sorting=none,backend=bibtex]{biblatex}
\newcommand\extrafootertext[1]{%
    \bgroup
    \renewcommand\thefootnote{\fnsymbol{footnote}}%
    \renewcommand\thempfootnote{\fnsymbol{mpfootnote}}%
    \footnotetext[0]{#1}%
    \egroup
}

\newcommand{\nine}{9$\times$9\xspace}

\newcommand{\nineteen}{19$\times$19\xspace}
\newcommand{\leelascore}{LZ-Score\xspace}
\newcommand{\lzs}{LZS\xspace}

\newcommand{\MMM}{\mathcal{M}} 
\newcommand{\SSS}{\mathcal{S}} 
\newcommand{\AAA}{\mathcal{A}} 
\newcommand{\PPP}{\mathcal{P}} 

\newcommand{\EE}{\mathbb{E}}

\newcommand{\rscore}{r_\textnormal{score}}
\newcommand{\rwinrate}{r_\textnormal{outcome}}
\newcommand{\MMMscore}{\mathcal{M}_\textnormal{score}}
\newcommand{\MMMwinrate}{\mathcal{M}_\textnormal{outcome}}
\newcommand{\policyscore}{\pi_\textnormal{score}}
\newcommand{\policywinrate}{\pi_\textnormal{outcome}}
\newcommand{\vscore}{v^\textnormal{score}}
\newcommand{\vwinrate}{v^\textnormal{outcome}}
\newcommand{\qscore}{q^\textnormal{score}}
\newcommand{\qwinrate}{q^\textnormal{outcome}}

\DeclareMathOperator*{\argmax}{argmax}
\DeclareMathOperator{\Var}{Var}

\theoremstyle{plain}

\newtheorem{proposition}{Proposition}

\newtheoremstyle{bfnote}%
{}{}%
{\bfseries}{}%
{\bfseries}{.}%
{.5em}{}%

\theoremstyle{bfnote}
\newtheorem*{claim}{Claim}

\begin{document}

\title{Score vs.\ Winrate in Score-Based Games: which Reward for Reinforcement Learning?}


\author{
\IEEEauthorblockN{Luca Pasqualini}
\IEEEauthorblockA{\textit{University of Siena, Italy}\\
\href{mailto:pasqualini@diism.unisi.it}{pasqualini@diism.unisi.it}}
\and
\IEEEauthorblockN{Gianluca Amato}
\IEEEauthorblockA{\textit{University of Chieti-Pescara, Italy}\\
\href{mailto:gianluca.amato@unich.it}{gianluca.amato@unich.it}}
\and
\IEEEauthorblockN{Marco Fantozzi}
\IEEEauthorblockA{\textit{Italy}\\
\href{mailto:marco.fantozzi@gmail.com}{marco.fantozzi@gmail.com}}
\and
\IEEEauthorblockN{Rosa Gini}
\IEEEauthorblockA{\textit{ARS Toscana, Italy}\\
\href{mailto:rosa.gini@ars.toscana.it}{rosa.gini@ars.toscana.it}}
\and
\IEEEauthorblockN{Alessandro Marchetti$^1$}
\IEEEauthorblockA{\textit{University of Chieti-Pescara, Italy}\\
\href{mailto:alessandro.marchetti@unicampus.it}{alessandro.marchetti@unicampus.it}}
\and
\IEEEauthorblockN{Carlo Metta}
\IEEEauthorblockA{\textit{ISTI-CNR Pisa, Italy}\\
\href{mailto:carlo.metta@isti.cnr.it}{carlo.metta@isti.cnr.it}}
\and
\IEEEauthorblockN{Francesco Morandin}
\IEEEauthorblockA{\textit{University of Parma, Italy}\\
\href{mailto:francesco.morandin@unipr.it}{francesco.morandin@unipr.it}}
\and
\IEEEauthorblockN{Maurizio Parton}
\IEEEauthorblockA{\textit{University of Chieti-Pescara, Italy}\\
\href{mailto:maurizio.parton@unich.it}{maurizio.parton@unich.it}}
}

\maketitle

\extrafootertext{Funded by INdAM groups GNSAGA, GNCS and GNAMPA. Computational resources provided by CLAI lab of Chieti-Pescara.}
\footnotetext[1]{Alessandro Marchetti is a PhD student enrolled in the National PhD in Artificial Intelligence, XXXVII cycle, course on Health and life sciences, organized by Campus Bio Medico University of Rome.}

\begin{abstract}

    In the last years, the DeepMind algorithm AlphaZero has become the state of the art to efficiently tackle perfect information two-player zero-sum games with a win/lose outcome.  However,  when the win/lose outcome is decided by a final score difference, AlphaZero may play score-suboptimal moves because all winning final positions are equivalent from the win/lose outcome perspective.  This can be an issue, for instance when used for teaching, or when trying to understand whether there is a better move.  Moreover, there is the theoretical quest for the perfect game.  A naive approach would be training an AlphaZero-like agent to predict score differences instead of win/lose outcomes. Since the game of Go is deterministic, this should as well produce an outcome-optimal play. However, it is a folklore belief that ``this does not work''.
    
    In this paper, we first provide empirical evidence for this belief. We then give a theoretical interpretation of this suboptimality in general perfect information two-player zero-sum game where the complexity of a game like Go is replaced by the randomness of the environment. We show that an outcome-optimal policy has a different preference for uncertainty when it is winning or losing. In particular, when in a losing state, an outcome-optimal agent chooses actions leading to a higher score variance. We then posit that when approximation is involved, a deterministic game behaves like a nondeterministic game, where the score variance is modeled by how uncertain the position is. We validate this hypothesis in AlphaZero-like software with a human expert.
\end{abstract}



\section{Introduction}
\label{sec:intro}

The game of Go has been a landmark challenge for AI research since its very beginning.  It is very suited to AI, with the importance of patterns and the need for deep exploration, and very tough to actually solve, with its whole-board features and subtle interdependencies of local situations. Nowadays, AI has reached a superhuman level in the game of Go with the well-known DeepMind algorithm AlphaGo~Zero~\cite{AlphaGoZero} (AGZ), a zero-knowledge evolution of AlphaGo~\cite{AlphaGo} (AG).  A more general version of the algorithm, AlphaZero~\cite{AlphaZero} (AZ), might even be able to tackle efficiently the whole class of perfect information two-player zero-sum games.

In perfect information two-player zero-sum games where the win/lose outcome is given by a final score difference and no other rewards during the game, maximizing this score difference is still an open and important question, see the detailed discussion in \cite[Introduction]{sai7x7}. In fact, AZ-like algorithms play suboptimal moves in the endgame, see for instance \cite[moves 210 and 214, page 252]{invisible}. The open-source clean room implementation of AGZ known as Leela Zero \cite{LeelaZero} (LZ) is also known to play suboptimal moves, see Section 4.4 in \cite{SAI9x9ECAI}.

This phenomenon is rooted in the win/lose reward in the reinforcement learning (RL) pipeline. Giving a reward of 1 (win) or 0 (lose) at the end of the game means that agents maximize just the winrate (that is, the expected win/lose outcome) and could play sub-optimally with respect to the expected score difference. It is a folklore belief that replicating the AZ pipeline using the score difference as a primary target, instead of the Boolean outcome, is unsuccessful.  A qualitative argument is that the score difference is unlikely to be a successful reward because without knowledge of the win/lose outcome, the agent will give the same importance to each score point.  But when the score difference is close to zero, a single point may change the outcome of the game, thus inducing instability in the agent strength. Note that this only happens because the perfect play is out of reach: if the training was to reach a perfect play, win/lose and score difference rewards would produce agents that are equivalent from the win/lose point of view, because the game of Go is deterministic.

As a matter of fact, there are at least two different RL approaches that proved somewhat successful in maximizing the score: KataGo \cite{KataGo} and SAI~\cite{SAI9x9ECAI}.  KataGo does include score estimation, but only as a secondary target: the value to be maximized is a linear combination of winrate and expectation of a nonlinear function of the score difference, not the score difference itself. In SAI, the winrate is modeled as a two parameters family of sigmoids $\sigma_{\alpha,\beta}$: while $\alpha$ can be seen as the final score difference, $\alpha$ and $\beta$ are learned indirectly by training $\sigma_{\alpha,\beta}$ against the classical Boolean reward.

Still, humans do use score estimations instead of win/lose outcome estimations while playing score-based games. The research question we address in this paper is to what extent a win/lose-based optimal play can be achieved by a score-based optimal play. In simpler words: if I play to maximize the score in a win/lose game, do I lose more often? 



This very question has different answers in different cases that are better-detailed in~\ref{sec:conceptual_framework}.  In a deterministic game sufficiently simple to allow for exact computation of the optimal play, maximizing the expected score difference yields also optimal win/lose-based play.  However, if the game is nondeterministic, maximizing the expected score difference does not always give optimality from the point of view of the winrate (even at optimal play).  We elaborate on this unexpected behavior and show that score variance is the key to understanding this suboptimality, see~\ref{sec:experimentsMDP}.  In a game that is deterministic but so complex that optimal play can only be roughly approximated, we support the mentioned folklore belief by training \leelascore (\lzs), an instance of LZ on the \nine board, using score difference as a target.  We show that the training is successful, see \ref{subsec:experiments_human}, but converges prematurely to a player weaker than a corresponding AGZ-like player, see \ref{subsec:experiments_sai}. Finally, in \ref{sec:background}, we detail the creation of \lzs.


\section{Conceptual framework}
\label{sec:conceptual_framework}


We consider two fully competitive agents playing in a perfect information finite sequential game, whose win/lose outcome (hereafter, simply \emph{outcome}) is decided solely by a final score difference. The game can be deterministic, with no chance involved like the game of Go, or nondeterministic with chance events, like backgammon. The agents can be score-based or outcome-based, according to whether they minimize/maximize the expected score difference or the expected win/lose outcome (hereafter, simply \emph{winrate}). For each agent, we call \emph{minimax optimal} any (it is not necessarily unique) theoretical optimal policy for that agent's target.  A backward induction argument shows that, in deterministic games, any minimax-optimal policy for the score-based agent would be also minimax optimal for the outcome-based agent, that is, to win in a win/lose game one can actually maximize the score.


However, in deterministic but very complex games where optimal policy is out of reach, maximizing the score may still lead to a suboptimal play from the winrate perspective. We empirically validate this claim by training \leelascore (\lzs), an AZ-like score-based agent, and by evaluating its performance with a human-in-the-loop and a quantitative approach, see~\ref{sec:experiments}.

One interpretation of this incoherence between score-based and outcome-based play is that in complex deterministic games, the value approximation is far from being optimal, and therefore the agent plays in partial ignorance. Partial ignorance can be modeled with the fact that, to the agent's eyes, the game is nondeterministic, and in this case, a minimax-optimal policy for the score-based agents is no longer guaranteed to be minimax optimal for the outcome-based agents.
In order to understand this phenomenon, we build a family of nondeterministic games where the minimax-optimal policy for the score-based agents is not optimal for the outcome-based agents, and elaborate that score variance is the key statistics to understand this behavior, see~\ref{sec:experimentsMDP}.


\section{Evaluation}
\label{sec:experiments}



\subsection{Qualitative evaluation against a human player}
\label{subsec:experiments_human}

Fifteen games were played between the best \lzs network and Carlo Metta, a strong amateur player\footnote{Player profile on EGD, the European Go Database: \url{https://www.europeangodatabase.eu/EGD/Player_Card.php?&key=14713996.}}.
Ten games were played with 400 visits, that is, the same setting as games in the quantitative evaluation described in~\ref{subsec:experiments_sai}, while five games were played with 20,000 visits, to test a stronger player.

A thorough analysis of such games shows that training has been successful in producing a consistent player, which, however, exhibits some unusual characteristics when compared to other artificial agents.
The match ended with a score of 14-1 in favor of the human player: although \lzs found itself in a position of clear advantage several times, it was only able to win one game, one of those with 20,000 visits.
\lzs showed some peculiar and not always desirable features. \lzs certainly has a direct and aggressive style. It does not seem to admit sacrificing a few stones for better final results, e.g.\ move 17 in game 6, nor to foresee sacrifice on the opponent's side, see e.g.\ move 16 in game 9. This is clearly in contrast with the flexibility shown by other artificial agents.

Another striking situation occurred several times: when in balanced positions, \lzs attempted to further increase the score difference, rather than settling for a narrow victory, in such an aggressive and self-delusional way that it resulted in an inevitable defeat. See for instance move 21 in game 8. It may be argued that this phenomenon was a direct effect of the \lzs training scheme.

\subsection{Quantitative evaluation against SAI}
\label{subsec:experiments_sai}

%

To estimate the Elo rating for \lzs networks promoted in training, we compared one every four of them against a calibration panel of 32 SAI networks, whose Elo ranged from 683 to 3501, with 400 visits.  Elo ratings were then estimated by maximum likelihood, as in R\'emi Coulom's Bayes Elo \cite{BayesElo}.


In Fig.~\ref{fig:calibrated_elo_sai29} we compare \lzs with run 1 of SAI, trained for a comparable number of self-plays, with the same network architecture and weights. see \cite{SAI9x9ECAI}.
This comparison shows that during training \lzs had consistently lower values of Elo. The last increase in the size of the network, at 900,000 games, did not yield any relevant improvement in the following 400,000 games, thus confirming that \lzs was converging prematurely to a weaker player. After 200,000 games the two curves are approximately parallel, with an Elo difference of around 1,500 points.  This is a remarkable difference in strength: the interpretation of the Elo formula means that, at the same level of training, \lzs would win against SAI with probability $10^{-(1500/400)}$, i.e. less than once every 5000 games.


\begin{figure}[t]
    \centering
    \includegraphics[scale=0.6]{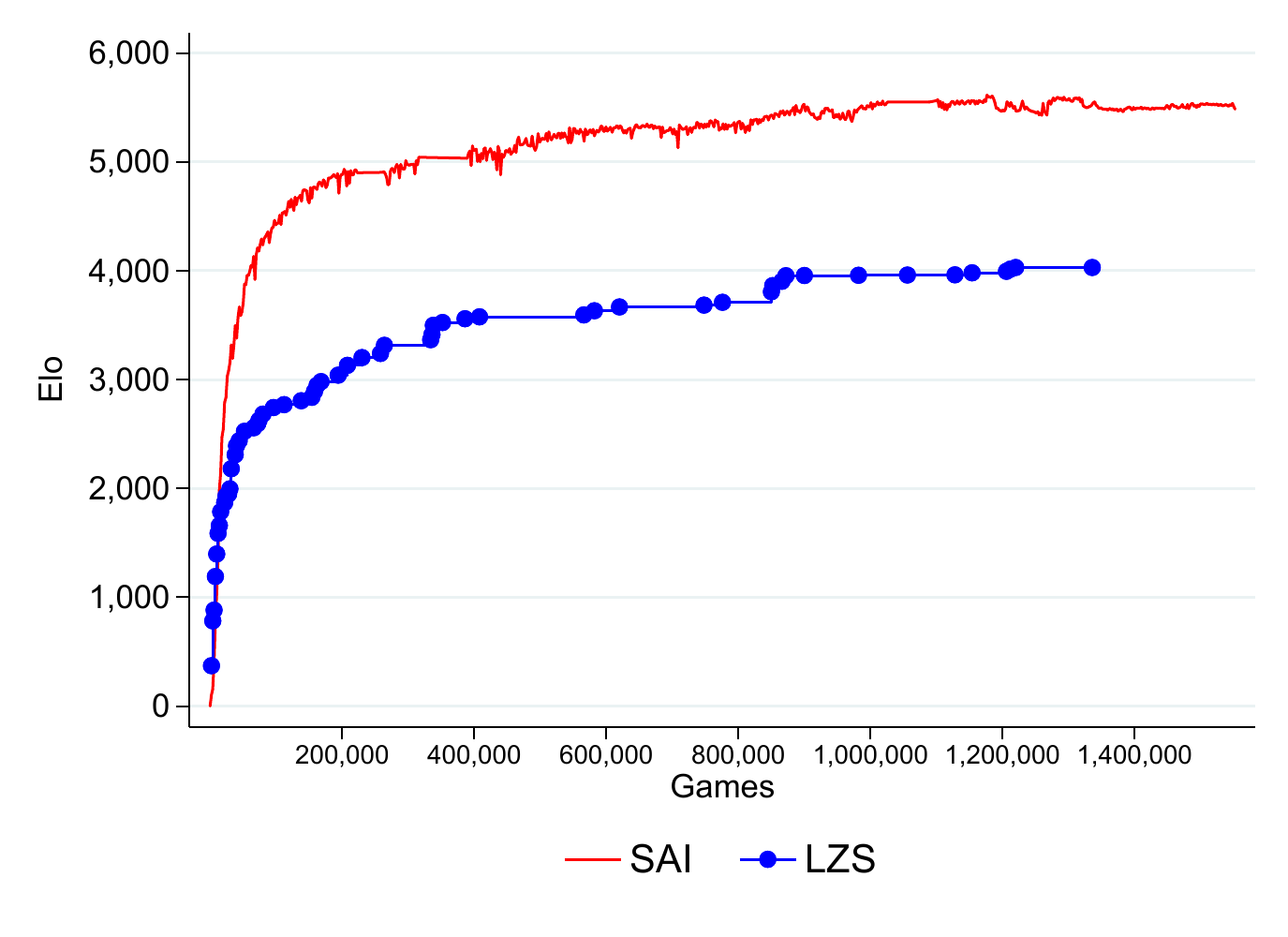}
    \caption{Elo ratings of all the promoted networks from \lzs and from run 1 of SAI \nine, as a function of the cumulative number of self-play games.}
    \label{fig:calibrated_elo_sai29}
\end{figure}

\section{Suboptimality of score-optimal policy in a win/lose outcome setting}\label{sec:experimentsMDP}

In this section we give a theoretical interpretation of the suboptimality of maximizing the score, considering two much simpler settings in which the complexity is replaced by the randomness.
In both settings, we show that suboptimality is a common behavior, and it is associated with the fact that the outcome-optimal policy has a different preference for uncertainty when it is winning or losing. When in a losing state, an outcome-optimal agent chooses actions leading to a higher score variance. According to strong human players, this is in fact what human players do too.

In the rest of this section, we consider a single agent instead of two fully-competitive agents.  By the analysis in \cite[Section 4.3]{Littman:1994}, this is not a restrictive hypothesis.




\subsection{Multi-armed bandit}

We start by proving that in the multi-armed bandit case, score-based optimality does not imply outcome-based optimality.

\begin{proposition}\label{prop:bandit}
    Consider a 2-armed bandit where each action $i \in \{ 1, 2 \}$ yields a random
	\emph{score} reward $\rscore(i) \in \mathbb{R}$, with Gaussian distribution $\mathcal{N}
	(\mu_i, \sigma_i^2)$ and an \emph{outcome} reward $\rwinrate(i) \in \{ 0, 1 \}$
	defined as $\rwinrate(i) := \mathbbm{1}_{\rscore(i) > 0}$.  Then there are choices of the parameters for which:
	\begin{equation}
	\label{eq:bandit_no_agree}
	    \mathbb{E}[\rscore(1)]>\mathbb{E}[\rscore(2)],
	    \mathbb{E}[\rwinrate(1)]<\mathbb{E}[\rwinrate(2)]
	\end{equation}
	Indeed this happens if and only if $\mu_1>\mu_2 > 0$ and $\sigma_1 > \sigma_2 \frac{\mu_1}{\mu_2}$, or $\mu_2<\mu_1 < 0$ and $\sigma_1 < \sigma_2\frac{\mu_1}{\mu_2}$.
\end{proposition}

\begin{proof}
    Compute $\mathbb{E}[\rwinrate(i)] = P (\rscore(i) > 0) = \Phi( \frac{\mu_i}{\sigma_i})$, where $\Phi$ denotes the cumulative distribution function. By the monotonicity of $\Phi$ and assuming~\eqref{eq:bandit_no_agree} we get $\mu_1>\mu_2$ and 
    $\frac{\mu_1}{\sigma_1} < \frac{\mu_2}{\sigma_2}$.  This implies that $\mu_1$ and $\mu_2$ must have the same sign, and from that, we deduce the thesis.
\end{proof}

We can interpret Proposition~\ref{prop:bandit} as a selective preference for lower or higher variance: if the two rewards have different optimal actions, then either the environment is ``winning'' (positive $\mu_i$'s and over 50\% probabilities of outcome 1), thus $\sigma_1 > \sigma_2$, and the \emph{outcome} reward prefers the lower variance action, or the environment is ``losing'', thus $\sigma_1 < \sigma_2$, and it prefers the higher variance action.

\subsection{Suboptimality in nondeterministic MDPs}

More generally, let $\MMM:=(\SSS,\AAA,r,\PPP)$ be a finite, episodic, and tabular Markov Decision Process (MDP), where $\SSS$ is the state space, $\AAA$ is the action space, $r$ is the reward function and $\PPP:\SSS\times\AAA\to\Delta(\SSS)$ is the transition model. To model a game where the win/lose outcome is given by a final score,
we make the additional assumption that $\SSS$ contains a subset of ``leaves'' $\SSS_L$, that is, states $s$ from which every action $a$ takes to the terminal state with probability $1$ and reward $r(s)$ depending on $s$ but not on $a$. The reward is received only at the end of the game: $r(s)=0$ for all $s\in\SSS-\SSS_L$. We denote by $\rscore$ this \emph{score} reward, and by $\MMMscore:=(\SSS,\AAA,\rscore,\PPP)$ the game where the reward is given by the score.

Correspondingly, $\MMMwinrate:=(\SSS,\AAA,\rwinrate,\PPP)$ denotes a game with the same state space, action space, and transitions as $\MMMscore$, but with a win/lose outcome. Here $\rwinrate$ is related to $\rscore$ by $\rwinrate(s) := \mathbbm{1}_{\rscore(s) > 0}$.

Any policy $\pi:\SSS\rightarrow\text{PD}(\AAA)$ induces value functions on $\MMMscore$ and $\MMMwinrate$. We denote these value functions by $\vscore_\pi$, $\vwinrate_\pi$,  $\qscore_\pi$ and $\qwinrate_\pi$. In particular, if $\policyscore$ and $\policywinrate$ are optimal policies for $\MMMscore$ and $\MMMwinrate$ respectively, then $\vwinrate_{\policyscore}(s)$ represents the probability of winning, starting from $s$ and following a score-based optimal play, in the game $\MMMwinrate$. With this notation, our research question becomes whether strict inequality may sometimes hold in
\begin{equation}
\label{eq:value_ineq}
    \vwinrate_{\policyscore}(s)
    \leq\vwinrate_{\policywinrate}(s)
    =\sup_\pi\vwinrate_{\pi}(s).
\end{equation}

\begin{claim}
In a nondeterministic MDP, $\policyscore$ can be suboptimal in $\MMMwinrate$.
\end{claim}

To prove this claim, we exhibit a class of MDPs that are exactly computable and for which the strict inequality can be numerically verified.  We consider tree-like MDPs with two parameters $b$ (state branch) and $d$ (depth), that we call \emph{action-shared tree} MDPs (an example in Fig.~\ref{fig:action_shared_tree_mdp}). More precisely, these trees are defined as follows: nonterminal states are partitioned by \emph{levels} $\ell=0,\dots,d$, with $b^\ell$ states at level $\ell$.  Each state $s$ has a unique set $C_s$ of $b$ ``children'' states at level $\ell+1$ to which one can transition, independently of the action.  From the $b^d$ leaves at level $d$, every action $a$ takes to the terminal state with probability $1$ and a reward independent of $a$.

\begin{figure}
\centering
\includegraphics[scale=0.2]{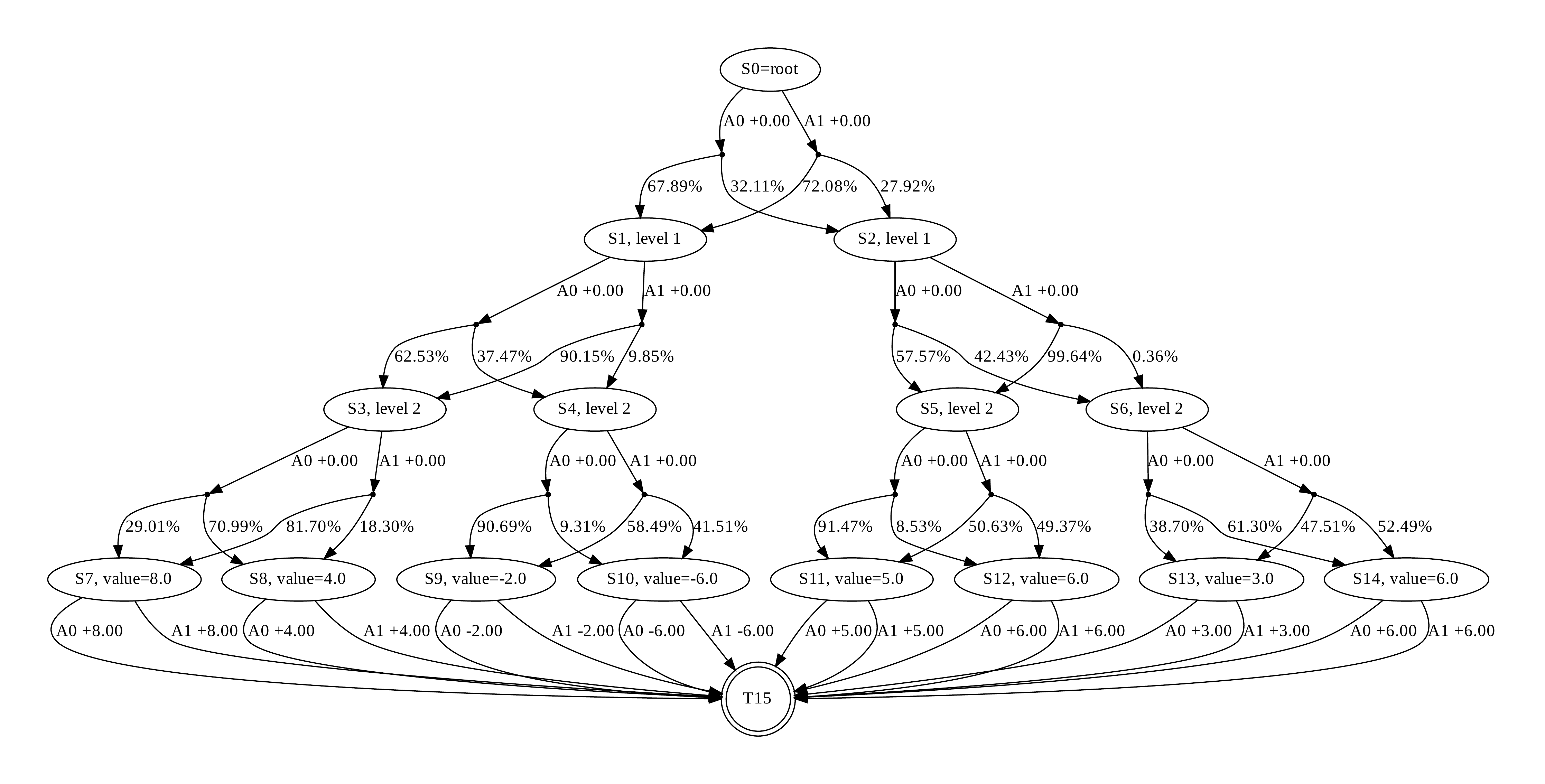}
\caption{A randomly generated action-shared tree MDP with state branch $b=2$, depth $d=3$ and 2 actions for each state. Numbers on edges denote rewards (nonzero rewards only on leaves) and transitions.}\label{fig:action_shared_tree_mdp}
\end{figure}

In our numerical experiments, we fixed the state branch $b$ and the depth $d$, and then we randomly generated the remaining parameters (all independently): the rewards $\rscore(s)$ for all leaves $s$, as integers taken uniformly in $[-b^d,b^d]$; the transition probabilities $(p(s'|s,a))_{s'\in C_s}$ for all nonterminal $(s,a)$, from a $b$-dimensional Dirichlet distribution with all parameters equal to 1, and $p(\cdot|s,a)=0$ outside $C_s$.

For each MDP so generated, we computed by value iteration the optimal policies $\policyscore, \policywinrate$ and the optimal value function $\vwinrate_{\policywinrate}$. Then, we computed $\vwinrate_{\policyscore}$ by policy evaluation.
We finally obtained the average of the difference $\vwinrate_{\policywinrate}(s)-\vwinrate_{\policyscore}(s)$ as a function of the level of the state, across the states of each MDP and across all the generated MDPs.  The results are depicted in Fig.~\ref{fig:winrate_difference} and show that there are MDPs and states for which~\eqref{eq:value_ineq} holds with the strict inequality, thus proving the claim.

\begin{figure}
		\hfil\includegraphics[scale=0.45]{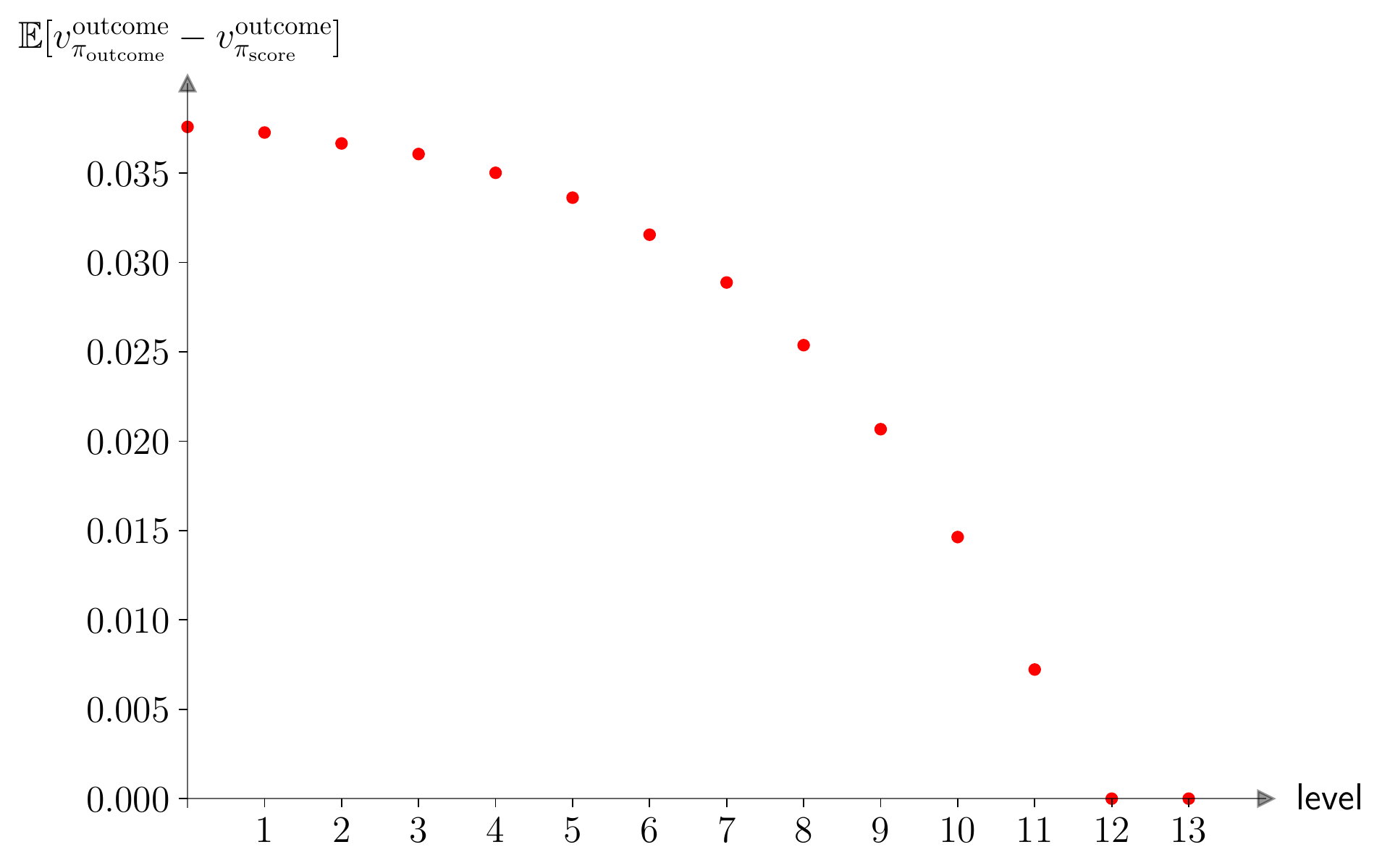}\hfil
	\caption{
	The average loss on 2,000 runs in the outcome value function $\vwinrate$ passing from the outcome-optimal policy to the score-optimal policy, as a function of the states level.}\label{fig:winrate_difference}
\end{figure}

\subsection{Fondness for the variance depends on the odds}

As seen for the multi-armed bandit setting, when the score-optimal policy $\policyscore$ and the win/lose outcome-optimal policy $\policywinrate$ differ, it is possible to interpret the differences in terms of variances of the scores of the available actions. A similar interpretation is possible in the more general MDP setting, as stated in the following claim.

\begin{claim}
When $\policywinrate$ is losing, it prefers actions leading to a higher score variance in $\MMMscore$, and vice versa when it is winning.
\end{claim}

Though this claim is too vague to be proved rigorously, we can support it with the following experiment, performed on the same simulated action-shared tree MDPs of the previous section.

\begin{enumerate}
	\item For assessing the preferences of the win/lose outcome-optimal policy, let $\pi^+:=\policywinrate$ and $\pi^-:=\policyscore$.
	\item For every state-action pair $(s,a)$, we compute the variance $\Var_{\pi^+}(\rscore|s,a)$, that is, the variance of the final score over trajectories following $\pi^+$ from $(s,a)$. This quantity can be computed efficiently as $\EE_{\pi^+}[\rscore^2|s,a]-(\qscore_{\pi^+}(s,a))^2$ by policy evaluation and value iteration.
	\item For every state $s$, let $y(s)$ be the difference of variance (in a log scale) between the actions preferred by $\pi^+(s)$ and $\pi^-(s)$, namely, $y(s):=\log(\Var_{\pi^+}(\rscore|s,\pi^+(s)))-\log(\Var_{\pi^+}(\rscore|s,\pi^-(s)))$.
	\item Consider only those states where $y(s)\neq 0$, that is, those states for which the action chosen by $\pi^+(s)$ and $\pi^-(s)$ actually differ. Draw a binned graph of the median of $y(s)$ against $\vwinrate_{\policywinrate}(s)$.  The median is computed across the states in the same $x$-axis bin in each MDP and across all generated MDPs.
	\item For assessing the preferences of the score-optimal policy restart from item 1, but with the role of $\pi^+$ and $\pi^-$ reversed: $\pi^+:=\policyscore$ and $\pi^-:=\policywinrate$.
\end{enumerate}

The results are depicted in Fig.~\ref{fig:variance_difference}, with $\policywinrate$ in red. The figure shows that when $\policywinrate$ chooses differently from $\policyscore$, if it is losing, then the chosen action will typically lead to a more uncertain game, with higher variance in the score. If it is winning, then the opposite happens. Moreover $\policyscore$, even computing variances following its own trajectories, agrees with the evaluation and typically chooses the opposite criterion.

\begin{figure}
        \centering
	\includegraphics[scale=0.5]{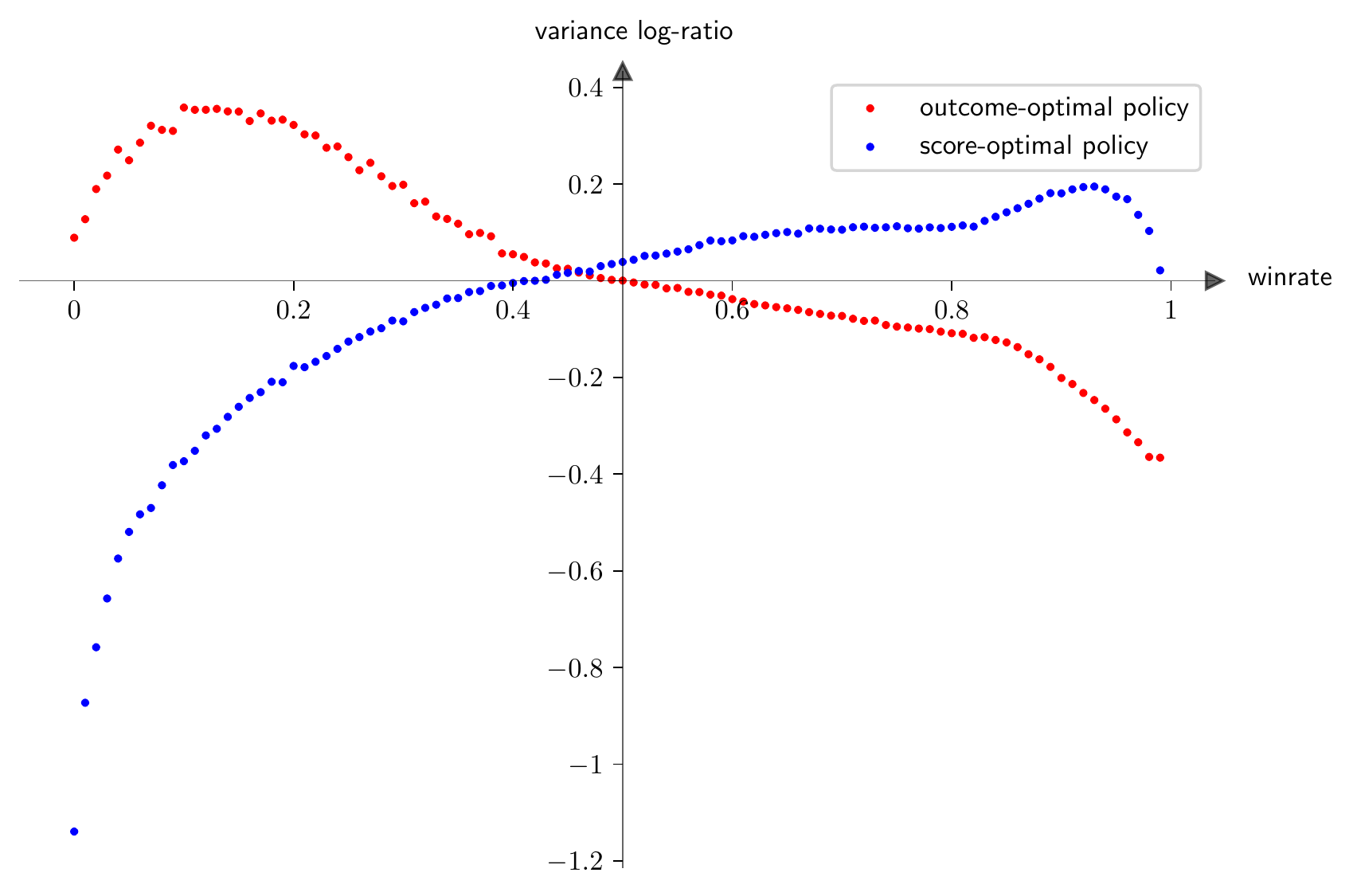}
	\caption{
	Log-ratio for the score variances of the chosen and discarded action, when the two policies do not agree (median over 100 bins). The two curves are from the point of view of either policy, with variances computed along their own trajectories. The $x$-axis is the best winrate, i.e.\ the outcome value from the point of view of the outcome-optimal policy.}\label{fig:variance_difference}
\end{figure}

\section{ \leelascore}
\label{sec:background}


Leela Zero (LZ)~\cite{LeelaZero} is an open-source Go program with no human-provided knowledge, known as one of the most faithful re-implementations of AlphaGo~Zero (AGZ)~\cite{AlphaGoZero}. For all intents and purposes, it is considered an open-source AGZ.

LZ was initially released on 25 October 2017. The neural networks powering the agent of LZ were trained by a distributed effort, which was coordinated at the LZ website: members of the community provided computing resources by running the client, which generates self-play games and submits them to the server.

The self-play games were used to train newer networks, which were then matched against the current best and possibly promoted according to a process known as gating.

The reinforcement learning pipeline of LZ ended on 15 February 2021. LZ is available at \url{https://zero.sjeng.org/} and \url{https://github.com/leela-zero/leela-zero}.

We developed \leelascore (\lzs) as a fork of LZ, suitably modified to change the reward from the outcome of the game to the score difference.  While Go is usually played on a \nineteen board, we trained \leelascore on the \nine board, which is still largely used for the game of Go and is of much simpler complexity.

SAI is a fork of LZ described in~\cite{sai7x7,SAI9x9extended,SAI9x9ECAI,SAIGitHub}. While LZ was trained on the \nineteen board only, SAI was trained on the \nine board, hence we choose SAI \nine as a benchmark to assess the strength of  \leelascore.  For the purposes of this work, SAI \nine can be considered an AGZ-like software.




\subsection{Setup of \leelascore}
\label{sec:lzs}

The overall architecture of LZ was replicated in \lzs, the main difference being in the quantity used as target and reward.  To use the integer score difference in the place of the Boolean game outcome, both during the training (as a target) and in the inference (as a reward), the following issues were addressed.
%
%

\begin{itemize}
    \item \textbf{Target for the neural network.}  The value head of LZ's neural network ends with a linear layer with a $\tanh$ activation and is trained against the outcome of the game in $\{-1,+1\}$, with an $l^2$ loss term; in \lzs with board size $N=9$, we removed the activation and changed the target to the score difference of the game, scaled by a factor $1/N^2$ in order to keep it inside $[-1,1]$.  The loss term was amplified by a factor of 20 to make its average size similar to the one of LZ.  The score difference was computed with Tromp-Taylor rules at the game completely finished (no resignation allowed).
    \item \textbf{Reward for the MCTS.}  For LZ, the exploration/exploitation tree search is driven by a UCT formula:
      \[
        a^*=\argmax_a(Q(a)+U(a))
        ,\quad
        Q(a)=\frac1{|V_a|}\sum_{i\in V_a}S(i)
      \]
      where $S(i)$ is the scaled score difference estimated by the value head for node $i$, $Q(a)$ is the mean action value over the set $V_a$ of the visited MCTS tree nodes that follow $a$, and $U(a)$ is the same as in~\cite{AlphaGoZero}, with the $c_{\text{puct}}$ constant raised from $1$ to $1.5$ to account for the larger variability of the scaled scores of siblings with respect to the winrate probabilities.
    \item \textbf{Game ending.} LZ employs a set of heuristics during self-plays to avoid uselessly long games. For example, most games are allowed to end by resignation, when the winrate drops below 5\%, and there are some hard-coded conditions for passing.  These heuristics do not work anymore in our new setting: to maximize the score, the agent should keep playing to the very end, so that the computationally amenable Tromp-Taylor score becomes equivalent to the proper area score of the game of Go.
\end{itemize}

Scaling down the board from size $N=19$ to size $\rho N=9$ with $\rho < 1$ yields several benefits.

\begin{itemize}
    \item Average number of legal moves at each position scales down by $\rho^2$. The number of good, meaningful moves might scale similarly.
    \item Average length of games scales down by $\rho^2$.
    \item The suggested number of visits in the MCTS tree might scale down by as much as $\rho^4$, as can be inferred from the previous two.
    \item The suggested number of layers in the residual convolutional neural network stack scales by $\rho$.
    \item The fully connected layers at the end of the neural network stack are smaller.
\end{itemize}

To summarize the performance benefits, the total speed improvement for self-play games can be estimated to be in the order of $\rho^9$, that is about $10^{-3}$. In fact, our resources allowed us to implement a pipeline with a larger number of visits.  We thus chose the same number of visits of \nine SAI, in order to grant a fair comparison.








\subsection{Reinforcement learning pipeline}
\label{sec:training}

Our training was composed of multiple phases, inspired by the original LZ training process, and by SAI well-documented pipeline.  In particular, we followed almost the same training procedure as in the first run of \nine~SAI~\cite{SAI9x9ECAI}, with the only exception that we reintroduced \emph{gating} (i.e.~the best network is replaced only if a stronger one emerges).  Gating was removed in the passage from AlphaGo Zero to AlphaZero, as it was seen as not being strictly necessary: since we are experimenting with score-based reinforcement learning, we considered it more conservative to keep it in our pipeline.  In fact, while we train the network to maximize its score, through gating we are assessing its capabilities in winning the game.  Each training cycle is called a generation, and it is composed of the following steps.
\begin{enumerate}
    \item \textbf{Self-play.} LZS engine uses the current best network to play a total of 2,000 \emph{self-play games} (against itself and sharing the tree search), with the following parameters: fixed number of visits $v$ depending on the generation (see the table below), random choice of the first 15 moves (with probability proportional to the visits number), Dirichlet noise at the root.  The agent is allowed to pass only if there is no other move with an expected score larger than the Tromp-Taylor score.
    \begin{center}
    \begin{tabular}{|c|c|c|c|c|c|}
    \hline
    generations & $v$ & $w$ & generations & $v$ & $w$ \\
    \hline
         0--15 & 100 & 4 &
        16--31 & 150 & 8 \\
        32--47 & 250 & 12 &
        48--63 & 400 & 16 \\
        64--79 & 600 & 20 &
        80--$\infty$ & 850 & 24 \\
    \hline
    \end{tabular}
    \end{center}

    \item \textbf{Training.} Starting from the current best network, new training is started over the self-play games of a window of the most recent $w$ generations (see the table above).  The number of steps adapts to the number of games inside the window, with a typical number of 30,000 steps. The batch size is 512.
    \item \textbf{Gating.} The newly trained network is matched against the current best, for a total of 400 games, with the same number of visits for self-play games and no randomness in the move choice.  The agent is allowed to pass when UCT selects this move.  If at least $55\%$ of the $400$ matches are won by the candidate network, it is promoted to be the new best network.
\end{enumerate}
All games were played with \emph{komi} $7.5$, that is, an additional $7.5$ points were added to white player's Tromp-Taylor score, to compensate for the fact that black plays the first move.

\subsection{Outcome of the learning pipeline}

The learning pipeline was guided by the intermediate results of the gating matches.  In fact, following LZ, at each generation we estimated the Elo rating of the new network from the 400 games in the gating match, anchoring arbitrarily the first random network to Elo 0.  This rating is clearly overrated (for example every promoted network is constrained to have at least 34.8 points, --corresponding to a 55\% winrate-- more than the previous one), and so we refer to it as \emph{uncalibrated Elo}.  This estimate was nonetheless adequate to assess when the learning pipeline stalled: when no new network overcame the gating for too many generations (empirically chosen), we scaled up the network size or decreased the learning rate of the training's optimizer.

Starting from a random $2\times64$ network (with 2 residual convolutional blocks of 64 filters), we moved to $4\times128$ at 2,000 games; to $8\times160$ at 154,000 games; to $10\times192$ at 748,000 games; to $12\times256$ at 982,000 games.  The learning rate started at 0.02 and decreased once to 0.002 at 850,000 games.
Fig.~\ref{fig:uncalibrated_elo} shows the pipeline results in terms of the uncalibrated Elo rating.




\begin{figure}[h]
    \centering
    \includegraphics[scale=0.5]{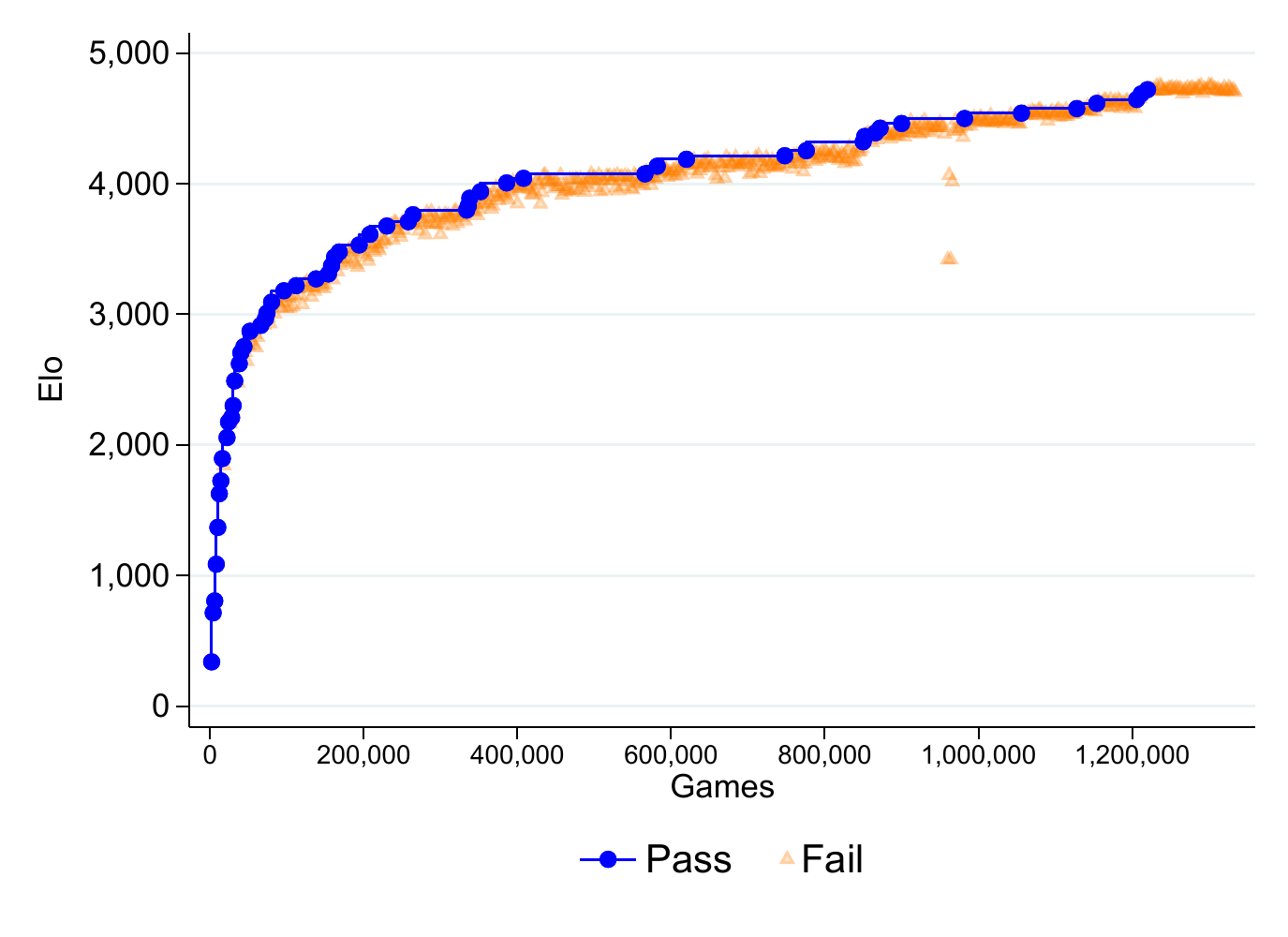}
    \caption{Uncalibrated Elo ratings of \lzs networks during the learning pipeline.  Blue circles and pink triangles represent networks that respectively passed and failed gating.  In this case, the rating of new networks was estimated only with the gating matches, yielding noisy and positively biased values.  See also Fig.~\ref{fig:calibrated_elo_sai29} for a more robust estimate.}
    \label{fig:uncalibrated_elo}
\end{figure}


\section{Conclusions}
\label{sec:conclusions}




In this paper, we discuss when and why maximizing the score in a score-based win/lose game can be suboptimal in terms of winrate. Since this cannot happen in a deterministic game, we start with the simplest nondeterministic case, that is, the multi-armed bandit. We prove that in the multi-armed bandit, a suboptimal behavior is related to the score variance and to whether the agent is winning or losing: when losing, the outcome-based agent prefers actions leading to a larger score variance, and vice versa.
We then define a class of MDPs modeling score-based win/lose games, and empirically show that this relation between suboptimal play and score variance is still valid. Finally, we train a score-based AlphaGo-like agent, and empirically show that it still behaves suboptimally, despite the fact that the game of Go is deterministic. This suggests that for certain problems, very complex deterministic games can be successfully modeled as nondeterministic, and provides sound and quantified evidence of the limitations of training a DRL score-based agent in a win/lose game, a folklore belief that we had not been able to find in the literature. 

\printbibliography

\end{document}